\documentclass[
hf
]{ceurart}

\usepackage{amsthm}
\newtheorem{definition}{Definition}
\newtheorem{proposition}{Proposition}
\newtheorem{example}{Example}
\newtheorem{corollary}{Corollary}

\newcommand{\dom}{\mathcal{D}}
\newcommand{\sourcedom}{\mathcal{D}^S}
\newcommand{\targetdom}{\mathcal{D}^T}
\newcommand{\reprspace}{\mathcal{R}}

\begin{document}

\copyrightyear{2022}
\copyrightclause{Copyright for this paper by its authors.
  Use permitted under Creative Commons License Attribution 4.0
  International (CC BY 4.0).}

\conference{IARML@IJCAI-ECAI'2022: Workshop on the Interactions between Analogical Reasoning and Machine Learning, at IJCAI-ECAI'2022,
  July, 2022, Vienna, Austria}

\title{Measuring the Feasibility of Analogical Transfer using Complexity}

\author[1]{Pierre-Alexandre Murena}[%
orcid=0000-0003-4586-9511,
email=pierre-alexandre.murena@aalto.fi,
url=http://pamurena.com/,
]
\fnmark[1]
\address[1]{Helsinki Institute for Information Technology HIIT, Department of Computer Science, Aalto University,
  Espoo, Finland}


\begin{abstract}
Analogies are 4-ary relations of the form ``A is to B as C is to D". While focus has been mostly on how to solve an analogy, i.e. how to find correct values of D given A, B and C, less attention has been drawn on whether solving such an analogy was actually feasible. In this paper, we propose a quantification of the transferability of a source case (A and B) to solve a target problem C. This quantification is based on a complexity minimization principle which has been demonstrated to be efficient for solving analogies. We illustrate these notions on morphological analogies and show its connections with machine learning, and in particular with Unsupervised Domain Adaptation. 
\end{abstract}


\maketitle

\section{Introduction}

Analogies are 4-ary relations of the form ``$A$ is to $B$ as $C$ is to $D$", denoted $A : B :: C : D$. Even though humans demonstrate strong capabilities of understanding and generating analogies, which has been intensively studied by cognitive sciences~\cite{GentnerH17}, these tasks are much more difficult for a machine. In particular, an important task consists in solving analogical equations: given $A, B$ and $C$, find $D$ such that $A : B :: C : D$ is a valid analogy. Solving such equations has been investigated in multiple domains: Boolean domains~\cite{MicletPrade}, formal concepts~\cite{BarbotMP19}, structured character strings~\cite{Hofstadter:1995:CPM:218753.218767}, semantic~\cite{mikolov2013linguistic,LimPR21} or morphological tasks~\cite{Lepage01}.

In most cases, the aforementioned methods are designed to provide an answer to any equation $A : B :: C : x$, regardless of whether the equation makes sense. This is not the case for humans, who consider that some analogies make more sense than others. 
Consider for instance the domain of Hofstadter analogies~\cite{Hofstadter:1995:CPM:218753.218767}. It describes analogies between character strings, with strong domain constraints relative to the order of the alphabet. For instance, the analogy ``ABC : ABD :: IJK : IJL", which is a typical illustrative example of this domain, is based on the notions of \emph{increment} and \emph{last element}.  
Intuitively, not all analogical equations have a solution in this domain: for instance, it is difficult for a human to find a satisfying solution to the equation~``ABC : HIC :: BFQ : $x$". This is confirmed by the results of the user study conducted by Murena et al.~\cite{murena2017complexity}.

Which analogical equations are indeed solvable, is a rather infrequently discussed question. However, it would have strong implications, be it in practical applications of analogies (e.g. in intelligent tutoring systems~\cite{murena2021inferring}) or in transfer learning. 
In this paper, we propose a first step toward this important direction, by considering how to measure the transferability of a source case $(A, B)$ to a target case $(C, D)$. To do so, we propose a very general formalization of the problem, which is shown to apply to transfer in both symbolic tasks (like morphological analogies) and numerical machine learning. Based on this formalization, and getting inspiration from applications of Kolmogorov complexity in inference~\cite{li2008introduction}, we propose several potential definitions of transferability and discuss their main properties.

\section{Preliminary Notions}
\label{sec:preliminary}

\subsection{Domains and Model Spaces}

Although various definitions of analogy have been proposed in the literature, we consider in this paper the following definitions, inspired by the recent framework of Antic~\cite{Antic2020}.

A domain~$\mathcal{D}$ is defined as the product of two spaces~$\mathcal{X}$ (the \emph{problem space}) and~$\mathcal{Y}$ (the \emph{solution space}). An element~$(x, y) \in \dom$ will be referred to as a \emph{case}.  

We introduce a set~$\mathcal{R}$ called \emph{representation space}. 
A model space~$\mathbb{M}_{\reprspace,\dom}$ of domain~$\dom$ based on representation~$\reprspace$ is defined as a subset $\mathbb{M}_{\reprspace,\dom} \subseteq \lbrace f: \mathcal{X} \times \mathcal{Y} \times \mathcal{R} \rightarrow [0,1] \rbrace$ of functions mapping a problem $x \in \mathcal{X}$, a solution $y \in \mathcal{Y}$ and a representation $r \in \mathcal{R}$, to a real number. Any $M \in \mathbb{M}_\dom$ is called a model of~$\dom$. When the context is clear, we will use the notation~$\mathbb{M}$ instead of $\mathbb{M}_{\reprspace,\dom}$. 


We illustrate these notions with two examples that will be further investigated in Sections~\ref{sec:morphological-analogies} and \ref{sec:uda}. 

\begin{example}[Recursive model for morphology]
\label{ex:morphology-model}
Given an alphabet~$\mathcal{A}$, we define by~$\mathcal{A}^*$ the set of words of~$\mathcal{A}$. 
The morphological domain consists of two forms of a word (e.g. declension of a word or conjugation of a verb in natural language): therefore, it is given by~$\dom = \mathcal{A}^* \times \mathcal{A}^*$.

We define the representation space~$\reprspace = \bigcup_{n=1}^\infty (\mathcal{A}^*)^n$. Let $\mathcal{F}$ be the space of all recursive functions~\footnote{i.e. functions computable by a Turing machine. See Section~\ref{sec:kolmogorov}. } from $\reprspace$ to $\dom$. For all~$\phi \in \mathcal{F}$, we define $M_\phi: \mathcal{A}^* \times \mathcal{A}^* \times \reprspace  \rightarrow \lbrace 0, 1 \rbrace$ as:
\begin{equation}
    M_\phi(x, y, r) = \begin{cases} 1 & \text{if } \phi(r) = (x, y) \\ 0 & \text{otherwise} \end{cases}
    \label{eq:morpho-model}
\end{equation}
We can then define the model space $\mathbb{M}_{\reprspace,\dom}$ as:
\begin{equation}
    \mathbb{M}_{\reprspace,\dom} = \left\lbrace M_\phi | \phi \in \mathcal{F} \right\rbrace
\end{equation}
\end{example}

This example has an easy interpretation. The morphological domain consists of two words $w_1$ and $w_2$, which are typically two flections of a same word. For instance, the tuple ``play : played" describes the flection of the English verb ``play" from the present tense to the perfect tense. Similarly, the tuple ``taloon : talossa" describes the flection of the Finnish noun ``talo" from the illative case to the inessive case. 

For a better readability, we decompose the recursive function~$\phi \in \mathcal{F}$ as $\phi(r) = (\phi_1(r), \phi_2(r))$. The function $\phi_1$ (resp. $\phi_2$) describes how the word $w_1$ (resp. $w_2)$ is formed based on some representation~$r$. For instance, in the ``play : played" example, we can have $\phi_1(r) = r$, $\phi_2(r) = r + \text{``ed"}$ (where the $+$ operation is the string concatenation), and both functions are instantiated with $r = \text{ ``play"}$.

\begin{example}[Probabilistic models on $\mathbb{R}^d$]
\label{ex:probabilistic-model}
We define the binary domain~$\dom = \mathcal{X} \times \mathcal{Y}$ with problem space~$\mathcal{X} = (\mathbb{R}^d)^*$ and solution space~$\mathcal{Y} = \mathcal{L}^*$, for some space~$\mathcal{L}$ of labels. When $\mathcal{L}$ is discrete, the problem is called \emph{classification}, otherwise \emph{regression}.  
An observation on domain~$\dom$ consists of one labelled dataset, where the problem is the unlabeled dataset (points in~$\mathbb{R}^d$) and the labels (in $\mathcal{L}$).

Let $\mathcal{P}$ be a set $\mathcal{P}_\Theta = \lbrace p_\theta: \mathcal{X} \times \mathcal{Y} \rightarrow [0,1] | \theta \in \Theta \rbrace$ of probability density functions on $\dom$ parameterized by~$\theta \in \Theta$. Fixing $\reprspace = \emptyset$ and therefore identifying $M(x, y, r)$ to $f(x,y)$, we can define the model space~$\mathbb{M}_{\reprspace,\dom} = \mathcal{P}_\Theta$. 
\end{example}

In this paper, we assume that all considered quantities are computable. For those which are not computable (for instance domains of real numbers used in Example~\ref{ex:morphology-model}), we will introduce relevant computable approximations.

\subsection{Kolmogorov Complexity}
\label{sec:kolmogorov}

Our framework relies on the use of Kolmogorov complexity~\cite{li2008introduction}. We propose a gentle introduction to this notion. In the following, we will use the notation~$\mathbb{B}$ to designate the binary set $\lbrace 0, 1 \rbrace$. 

A function $\phi: \mathbb{B}^* \rightarrow \mathbb{B}^*$ is called partial recursive if its output $\phi(p)$ corresponds to the output of a given Turing machine after its execution with input~$p$ when it halts (otherwise, we use the convention $\phi(p) = \infty$). With this notation, the function~$\phi$ can be improperly likened to a Turing machine. In this case, the input~$p$ is called a \textit{program}. 
A partial recursive function~$\phi$ is called \textit{prefix} if, for all~$p, q \in \mathbb{B}^*$, if $\phi(p) < \infty$ and $\phi(q) < \infty$, then $p$ is not a proper prefix of $q$. 

Complexity~$K_{\phi}(x)$ of a string $x \in \mathbb{B}^*$, relative to a partial recursive prefix (p.r.p.) function~$\phi$, is defined as the length of the shortest string $p$ such that $\phi(p) = x$: 
\begin{equation}
K_{\phi}(x) = \min_{p \in \mathbb{B}^*} \lbrace l(p) : \phi(p) = x \rbrace
\end{equation}
where $l(p)$ represents the length of the string $p$. 

A key result of the theory of complexity is the existence of an additively optimal p.r.p. function~$\phi_0$: for any p.r.p. function~$\phi$, there exists a constant~$c_{\phi}$ such that for all $x \in \mathbb{B}^*$, $K_{\phi_0}(x) \leq K_{\phi}(x) + c_{\phi}$. 
These additively optimal p.r.p. functions are used to define Kolmogorov complexity. They present in particular invariance properties, which means that the difference between the complexities defined by two distinct universal p.r.p. functions is bounded. However, it can be shown that Kolmogorov complexity is not computable, and thus cannot be used in practice.

In practice, this limitation is overcome by fixing a reference p.r.p. function which is not optimal but leads to a computable complexity. By definition, this non-optimal complexity is an upper-bound of Kolmogorov complexity (up to an additive constant).
This choice of a reference function is particularly restrictive and imposes some biases, which is inherent to any inductive problem. 
In the following, we will refer to this upper-bound either as complexity or as description length.

\subsection{Inference on a Single Domain}

Consider a domain~$\dom = \mathcal{X} \times \mathcal{Y}$ and an observation $(x, y) \in \dom$. Given a model space~$\mathbb{M}_\dom$, the standard \emph{inference} task of supervised learning is to select a model $M \in \mathbb{M}_\dom$ which optimally describes the observation. 

Selecting a model which accurately describes the observation is not enough in general: In most situations there exists multiple such models, and sometimes even infinitely-many. It is then necessary to discriminate among all these models, in particular based on the intended use of the model. In statistical learning for instance, the discrimination is done by both the selection of a \emph{simple} model family and a \emph{penalization} of models~\cite{shalev2014understanding}. 

The inference of the model describing an observed case can then be split into two aspects: selection of a model accurately describing the case and penalization over the model space. This idea has been formalized by Algorithmic Information Theory using Kolmogorov complexity. The \emph{Minimum Message Length} (MML)~\cite{wallace1968information} and the \emph{crude Minimum Description Length} (MDL)~\cite{rissanen1978modeling} principles both formulate the general inference task over a domain as the following minimization problem: 
\begin{equation}
    \underset{M \in \mathbb{M}_\dom}{\text{minimize}} \qquad K(M) + K(x,y | M)
    \label{eqn:mdl-onedomain}
\end{equation}
This two-part objective function corresponds to the trade-off between the accuracy of the model and its simplicity. For instance, in the morphological domain, it is possible to define a model accounting for all possible valid transformations. By construction, this model will have perfect accuracy, but it will require to encode every pair of problems and solutions in the language, and therefore will be particularly complex.

We remind the computation of complexity is relative to a choice of a reference Turing machine. Therefore, this choice imposes a strong bias over the intended outcome. A \emph{refined} version of the MDL principle has been proposed, which overcomes this limitation. This version is beyond the scope of this paper, but we refer the interested reader to (Gr\"{u}nwald, 2007)~\cite{grunwald2007minimum}.

\section{A Definition of Transferability}
\label{sec:transferability}

Analogies involve two separate domains: a \emph{source domain}~$\sourcedom$ and a \emph{target domain}~$\targetdom$. Given a source problem-solution pair $(x^S, y^S) \in \sourcedom$ and a target problem $x^T \in \mathcal{X}^T$, the analogical transfer from $(x^S,y^S)$ to $x^T$ is informally defined as finding $y^T \in \mathcal{Y}^T$ such that the transformation $x^S \mapsto y^S$ is ``similar" to the transformation $x^T \mapsto y^T$. This notion of similarity is problematic, especially when the source and target domains are distinct. Existing frameworks of analogy often assume that the source and target domains are the same~\cite{MicletPrade}, or at least share a common structure (e.g. are $L$-algebras of a same language~$L$, such as proposed by (Antic, 2020)~\cite{Antic2020}).

In this section, we show how complexity can be used to properly define this similarity, even in the case of distinct domains. We will then show that this definition helps quantifying the notion of \emph{transferability}, i.e. how the source observation $(x^S,y^S)$ is useful to find a solution to target problem~$x^T$.

\subsection{Inference of a Target Model}

The task in the target domain consists in predicting the solution~$y^T \in \mathcal{Y}^T$ associated to the problem~$x^T \in \mathcal{X}^T$. Usually, this is done throughout a model~$M$, in particular by finding $y^T \in \mathcal{Y}^T$ and $r \in \reprspace$ maximizing the score $M(x^T,y^T,r)$:
\begin{equation}
    y^{T*} \in \text{arg}\max_{y \in \mathcal{Y}^T} \left\lbrace \max_{r \in \mathcal{R}} M(x^T, y, r) \right\rbrace
     \label{eq:model-based-prediction}
\end{equation}
In the context of Example~\ref{ex:morphology-model}, this corresponds to choosing a representation $r$ that successfully describes~$x^T$ given the recursive function~$\phi$ associated to~$M$, i.e. finding $r$ such that $\phi_1(r) = x^T$. The solution~$y^T$ is then estimated by taking $y^T = \phi_2(r)$. In the context of Example~\ref{ex:probabilistic-model}, Equation~\eqref{eq:model-based-prediction} corresponds to taking the most probable solution.  

However, in practice, the model~$M$ is not known and needs to be inferred. The difficulty is that, in general, it is not possible to identify~$M$ given $x^T$ only. Even worse, there is no guarantee that two models correctly describing~$x^T$ can yield the same values of~$y^T$. For instance, in the morphological domain (Example~\ref{ex:morphology-model}), one can build degenerate functions~$\phi$ such that $\phi(r) = x^T$ for all $r$. All such functions successfully describe~$x^T$ and yield all possible values for~$ y^T$. 

The MML principle presented in Equation~\ref{eq:model-based-prediction} is meant to be used in a single domain, in particular the source domain. We propose to use it as well to estimate the target model $M^T$. However, in the context of an analogy, some additional information is provided by the observation of the source domain, where a model~$M^S$ can be evaluated from observations~$(x^S, y^S) \in \sourcedom$. The target model~$M^T$ can then be described with regards to the source model~$M^S$. 

The limitation of such a relative description of models is that the source model may not be relevant. In the following, we propose to quantify this relevance with two measures of model reusability,

\subsection{Reusability of a Source Model}

We measure the reusability of a source model to its ability to compress the information about the target domain. We identify two main possible definitions, that we call \emph{weak} and \emph{strong} reusability, and will show that a strongly reusable model is necessarily weakly reusable. 

The main criterion for \emph{weak reusability} is that knowing the source model~$M^S$ makes the target inference better, in the sense that it compresses more the description of the target case~$(x^T, y^T) \in \targetdom$. 
 
\begin{definition}[Weak reusability]
Let $\eta > 0$.
A model~$M^S \in \mathbb{M}_{\mathcal{R}^S, \dom^S}$ is called \emph{weakly $\eta$-reusable} for case $(x^T, y^T)$ in target model space~$\mathbb{M}_{\mathcal{R}^T, \dom^T}$ if:
\begin{align}
\min_{M \in \mathbb{M}_{\mathcal{R}^T, \dom^T}} \lbrace & K(M) + K(x^T, y^T | M) \rbrace \nonumber \\
& \geq \eta + \min_{M \in \mathbb{M}_{\mathcal{R}^T, \dom^T}} \lbrace K(M | M^S) + K(x^T, y^T | M) \rbrace \label{eq::weak-reusability}
\end{align}
\label{def:weak-reusability}
\end{definition}

It is essential to keep in mind that this definition is relative to the choice of a model space~$\mathbb{M}_{\mathcal{R}^T, \dom^T}$ for the target domain. A source model~$M^S$ may not be reusable for a case $(x^T, y^T)$ depending on the chosen target model space. In the following, we will abusively omit to mention the target model space, for readability purposes. 

Note that the models~$M$ defined in the left-hand side and in the right-hand side of inequality~\eqref{eq::weak-reusability} are not the same. On the left-hand-side, the model corresponds to the optimal target model describing~$(x^T, y^T)$ while, on the right-hand-side, it is the optimal model describing~$(x^T, y^T)$ when $M^S$ is known. 

This notion of reusability means that providing the source model~$M_S$ helps finding a new description of the problem shorter than the optimal description by~$\eta$ bits. 
In particular, in the case where $\mathbb{M}_{\mathcal{R}^S, \dom^S} = \mathbb{M}_{\mathcal{R}^T, \dom^T}$, the optimal model for $(x^T, y^T)$ (i.e. the model that minimizes $K(M) + K(x^T, y^T|M)$) is trivially weakly reusable for $(x^T, y^T)$ for all~$\eta$. In other words, when the optimal model for the source domain is also optimal for the target domain, then it is obviously reusable for the target domain. The interesting case will be when the optimal source model is not optimal for the target domain. 

We notice that the definition does not require $\sourcedom = \targetdom$ and aims to quantify the reusability of a model even for a completely different task. This is possible since complexity only requires to have computable models: Indeed, the term $K(M | M^S)$ is defined as long as the models are computable. 
In practice, the issue of comparing objects of different nature is hidden within the choice of the reference p.r.p. function for complexity (see Section~\ref{sec:kolmogorov}). 

\medskip

We propose an alternative definition of reusability, called \emph{strong reusability}. 
It is based on the idea that $M_S$ is reuseable if it helps compressing the optimal model of target case~$(x^T, y^T)$. Unlike previous definition, $M_S$ is not directly involved in the description of~$(x^T, y^T)$ though. 

\begin{definition}[Strong reusability]
Let $\eta > 0$.
A model~$M^S \in \mathbb{M}_{\mathcal{R}^S, \dom^S}$ is called \emph{strongly $\eta$-reusable} for case $(x^T, y^T)$ in target model space~$\mathbb{M}_{\mathcal{R}^T, \dom^T}$ if:
\begin{align}
M \in \underset{M \in \mathbb{M}_{\mathcal{R}^T, \dom^T}}{\text{arg} \min} \lbrace & K(M) + K(x^T, y^T | M) \rbrace \nonumber \\
& \Longrightarrow K(M) \geq K(M | M^S) + \eta \label{eq:strong-reusability}
\end{align}

\label{def:strong-reusability}
\end{definition}

This definition also relies on the choice of a target model space. As for weak reusability, we will abusively omit the model space in the following notations. 

Strong reusability is an extremely strong property of a source model. Indeed, it would be a natural property that, in case~$\mathbb{M}_{\mathcal{R}^S, \dom^S} = \mathbb{M}_{\mathcal{R}^T, \dom^T}$, any model minimizing $K(M) + K(x^T, y^T | M)$ is reusable to~$(x^T, y^T)$. This is not the case for strong reusability: indeed, the definition implies that \emph{any} model minimizing $K(M) + K(x^T | M)$ can be compressed given $M^S$, not only $M^S$ itself. Note that this could be alleviated by weakening Definition~\ref{def:strong-reusability} and requiring only the existence of a model $M$ minimizing $K(M) + K(x^T, y^T | M)$ and such that $K(M) \geq K(M | M^S) + \eta$.

\subsection{Properties of Reusability}

We now present basic properties of these two notions of reusability. All the presented properties follow directly from the definitions.

The first property applies to both weakly and strongly reusable models: it states that the threshold~$\eta$ is not unique. The proof of this proposition is trivial and omitted. 

\begin{proposition}
Let~$M_S \in \mathbb{M}_{\mathcal{R}^S, \dom^S}$ and $(x^T, y^T) \in \targetdom$. If $M^S$ is $\eta$-reusable for $(x^T, y^T)$, then it is also $\eta^\prime$-reusable for $(x^T, y^T)$ for all $\eta^\prime \leq \eta$.
In the following, we will call \emph{degree of reusability} of $M^S$ to $(x^T, y^T)$ the quantity:
\begin{equation}
    \rho(M^S, (x^T, y^T)) = \max \left\lbrace \eta \  ; \ M^S \text{ is } \eta \text{-reusable for } (x^T, y^T) \right\rbrace
    \label{eq:degree-reusability}
\end{equation}
\label{prop:reusability-monotonicity}
\end{proposition}

However, we insist on the fact that the degree of reusability is relative to a chosen target model space. It also depends on which of weak or strong reusability is considered. 
This dependency would not exist if these two notions of resuability were equivalent, which we will see is not the case. We will use the notation $\rho_w$ and $\rho_s$ to specify between the weak and strong cases.

We now show that strong reusability implies weak reusability: 

\begin{proposition}
Let $M^S \in \mathbb{M}_{\mathcal{R}^S, \dom^S}$ a source problem, a target case $(x^T, y^T) \in \targetdom$ and $\eta > 0$.
If $M^S$ is strongly $\eta$-reusable for $(x^T, y^T)$, then $M^S$ is also weakly $\eta$-reusable for $(x^T, y^T)$. 
\label{prop:correlation-reusability}
\end{proposition}

\begin{proof}
We call~$M^* \in \mathbb{M}_{\mathcal{R}^T, \dom^T}$ an optimal model for $(x^T, y^T)$: $M^* \in \text{arg} \min_M \lbrace K(M) + K(x^T, y^T | M) \rbrace$. 
Under the assumptions of the proposition, it follows that:
\begin{align*}
\min_M \lbrace &K(M) + K(x^T, y^T | M) \rbrace \\
& = K(M^*) + K(x^T, y^T | M^*) \\
&\geq K(M^* | M^S) + \eta + K(x^T, y^T | M^*) \\
&\geq \min_M \lbrace K(M | M^S) + K(x^T, y^T | M) + \eta \rbrace
\end{align*}
which proves the proposition.
\end{proof}

However, the converse is not true: weak reusability does not imply strong reusability. This is the consequence of the fact that the models implied in Equation~\eqref{eq::weak-reusability} are not the same on the right hand side and on the left-hand side. We illustrate this with a simple example.

\begin{example}
\label{ex:weak-notstrong}
We consider a target model space made up of two distinct models: $\mathbb{M}_{\mathcal{R}^T, \dom^T} = \lbrace M_1, M_2 \rbrace$. We assume the following properties:
\begin{itemize}
    \item $M_1$ and $M_2$ describe equally well case~$(x^T, y^T)$, in the sense that $K(x^T, y^T | M_1) = K(x^T, y^T | M_2)$
    \item Model $M_1$ is more complex than model~$M_2$: $K(M_1) > K(M_2)$
    \item Model $M_1$ is easily described by source model~$M^S$: for simplicity, we can take~$K(M_1 | M^S) = 0$
    \item Model $M_2$ is not well described by source model~$M^S$: for simplicity, we can take~$K(M_2 | M^S) = K(M_2)$
\end{itemize}

With these assumptions, it can be easily verified that $M^S$ is weakly $\eta$-reusable for~$(x^T, y^T)$, with $\eta \leq K(M_2)$. However, $M^S$ is not strongly $\eta$-reusable for~$(x^T, y^T)$, since $M_2$ minimizes $K(M) + K(x^T, y^T | M)$ but $K(M_2) < K(M_2) + \eta = K(M_2 | M^S) + \eta$. 

\end{example}

Putting together the results of Propositions~\ref{prop:reusability-monotonicity} and \ref{prop:correlation-reusability}, as well as the counter-example of Example~\ref{ex:weak-notstrong}, we can establish the following result:
\begin{corollary}
For~$M_S \in \mathbb{M}_{\mathcal{R}^S, \dom^S}$ and $(x^T, y^T) \in \targetdom$:
\begin{equation}
    \rho_s(M^S, (x^T, y^T)) \leq \rho_w(M^S, (x^T, y^T))
\end{equation}
\end{corollary}

\subsection{Transferable Cases}

The notion of  reusability we proposed in Definitions~\ref{def:weak-reusability} and~\ref{def:strong-reusability} are not directly applicable to measure transferability from a source observation to a target problem. The property of transferability measures the ability to transfer knowledge from a source case~$(x_S, y_S) \in \sourcedom$ to apply it to the target case~$(x^T, y^T) \in \targetdom$. It can be seen as an extension of reusability where the source model~$M^S$ is determined based on the source observation. 

\begin{definition}[Transferability]
\label{def:transferability}
Let $(x^S, y^S) \in \sourcedom$ be a source case. The case~$(x^S, y^S)$ is said to be strongly (resp. weakly) $\eta$-transferable to the target case~$(x^T, y^T) \in \targetdom$ if the set of compatible models $\lbrace M^S | K(M^S) + K(x^S, y^S | M_S) < K(x^S, y^S), M^S \in \mathbb{M}_{\mathcal{R}^S, \dom^S} \rbrace$ contains an element $M^{S*}$ such that $M^{S*}$ is strongly (resp. weakly) reusable for case~$(x^T, y^T)$. 
\end{definition}

The proposed definition of transferability is very weak, since it only requires the existence of one source model that is reusable. However, it does not take into account the quality of this model in the source domain. For instance, it would acquire equal weight if the reusable source model is associated with complexity $K(M^S) + K(x^S, y^S | M^S)$ close to $K(x^S, y^S)$, as if the complexity is close to 0. Such situations are very likely to occur with probabilistic models, since many such models give positive probability to all observations.

In order to refine this notion, it would be important to define a coefficient of transferability, similar to the degree of reusability~$\rho$ defined in Proposition~\ref{prop:reusability-monotonicity}. We will denote such a coefficient $\tau((x^S, y^S), (x^T, y^T))$. We propose two possible definitions below. These definitions rely on the set of compatible models for $(x^S, y^S)$ introduced in Definition~\ref{def:transferability}, and denoted~$\mathbb{M}(x^S, y^S)$. For simplicity, we will use the notation $\mathcal{C}^S$ (resp. $\mathcal{C}^T$) to designate the case $(x^S, y^S)$ (resp. $(x^T, y^T)$). 

A first definition is a direct application of Definition~\ref{def:transferability}: it associates the transferability of a problem to the maximum reusability of the corresponding model:
\begin{equation}
\tau_{\text{max}}\left(\mathcal{C}^S, \mathcal{C}^T\right) = \max_{M^S \in \mathbb{M}(\mathcal{C}^S)} \rho(M^S, \mathcal{C}^T) 
\label{eq:transferability-max}
\end{equation}
with the convention that $\max \emptyset = 0$. 
This definition has the same weakness as the introduced concept of transferability: it does not take into account that the most reusable model might also be very weak to describe $(x^S, y^S)$. 

In order to alleviate this, our second definition proposes to average the score over the possible models. Therefore, we compute the posterior of the model for given~$(x^S, y^S)$ and compute the score as the expected value over~$M_S$ of the degree of reusability. 
\begin{equation}
\tau_{\text{avg}}\left(\mathcal{C}^S, \mathcal{C}^T\right) = \sum_{M^S \in \mathbb{M}(\mathcal{C}^S)} p(M^S | x^S, y^S) \rho(M^S, \mathcal{C}^T) 
\label{eq:transferability-avg}
\end{equation}
In order to compute the posterior~$p(M^S | x^S, y^S)$, we use algorithmic probability~\cite{li2008introduction} and assume a uniform prior:
\begin{equation}
    p(M^S | x^S, y^S) = 2^{- K(x^S, y^S | M^S) - K(M^S)}
\end{equation}
This second definition provides a better idea of how transferable a source observation can be, since it takes into account all possible models describing it. However, it has two major weaknesses. On a theoretical level, it does not reflect the variance of the reusability degree over the compatible models: based on $\rho_{\text{avg}}$ only, it is impossible to know whether all models, only very few but very compatible models, or a large number of less compatibles are reusable. On a practical level, the $\tau_{\text{avg}}$ score may not be tractable, and would require some approximations.




\section{Illustration: Transferability of Morphological Transformations}
\label{sec:morphological-analogies}

In this section, we propose to apply the notions introduced in Section~\ref{sec:transferability} to the case of morphological analogies. We will mostly build upon the domain introduced in Example~\ref{ex:morphology-model}. 

\subsection{Introduction to Morphological Analogies}

Morphological analogies are analogies on words involving transformations of a morphological nature, for instance declension or conjugation. Unlike semantic analogies (e.g. ``king is to queen as man is to woman"), morphological analogies are mostly of a symbolic nature: they involve the detection of transformations of one form of a word into another form. Typical example of such morphological analogies could be ``work : worked :: call : called" in English (transformation from present to preterit, in English), or ``voihin : vuossa :: soihin : suossa" (transformation from illative plural to inessive singular in Finnish). 

Various works have been proposed to solve such analogies, mostly based on the principles of proportional analogy~\cite{Lepage01}. Other approaches rely on an algebraic consequence of these principles~\cite{langlais2009improvements}, on deep learning approaches~\cite{alsaidia} or even on complexity minimization~\cite{murena2020solving}.

\subsection{A Simplified Model for Morphology}

The model proposed in Example~\ref{ex:morphology-model} is relevant for a formal treatment of morphological analogies. However, following Murena et al.~\cite{murena2020solving}, we propose a simplification where the space of partial recursive functions is restricted to a simpler subset. This subset is defined by a simple descriptive language based on the concatenation of various strings. It allows to define functions of the form $\phi = (\phi_1, \phi_2)$ with, for $i \in \lbrace 1, 2 \rbrace$:
\begin{equation}
    \phi_i(r_1, \dotsc, r_n) = w_0^i + \sum_{k=1}^{K_i} r_{\sigma_i(k)} + w_k^i
\end{equation}
where the + operation stands for string concatenation, $r_1, \dotsc, r_n \in \mathcal{A}^*$ and $w_0^i, \dotsc, w_K^i \in \mathcal{A}^*$ are words on the alphabet~$\mathcal{A}$, and $\sigma_i: \lbrace 0, \dotsc, K_i \rbrace \rightarrow \lbrace 0, \dotsc, n \rbrace$. The vector $r = (r_1, \dotsc, r_n)$ corresponds to the representation, and the models are defined such as in Equation~\eqref{eq:morpho-model}. 

We notice that the proposed restriction accounts for various types of morphological transformations, such as prefixation, suffixation, change or prefix and/or suffix, but also duplication. However, the language is not Turing complete, and for instance does not cover any conditional statement. 

In morphological analogies, the source and target domains are the same.

\subsection{Computing Complexities}
\label{sec:morphology-complexities}

In order to compute the reusability and transferability, we must define three expressions of complexity: complexity of a model~$K(M)$, complexity of a case given a model~$K(x, y | M)$ and complexity of a target model given a source model~$K(M^T | M^S)$. 

A model is entirely defined by the function~$\phi$. A binary coding of such functions is proposed in~\cite{murena2020solving}. The complexity of the model then corresponds to the length (in bits) of the binary code of the function. 

Given a model~$M$, the case~$(x, y)$ is coded by providing a correct representation~$(r_1, \dotsc, r_n)$. In case no representation generates~$(x,y)$ with model~$M$, the two words are hard-coded. In terms of binary representation, we propose the following coding. The string starts with a bit encoding whether the following bits code for the representation vector or for the two words. After this bit, the words are coded letter by letter, with specific delimitors to mark the end of each string. The complexity~$K(x, y | M)$ is the length of this code. 

For the model transfer, we propose an elementary description. More sophisticated versions have to be discussed in future works. We propose to reuse the source model either by reusing it directly (without modification), or by redefining completely. Such as previously, this choice is indicated by an initial bit. The model complexity is then given by:
\begin{equation}
    K(M^T | M^S) = \begin{cases} 
        1 & \text{ if } M^T = M^S \\
        1 + K(M^T) & \text{ otherwise}
    \end{cases}
\end{equation}

\subsection{Examples of Reusability}

\paragraph{Suffixation.} We consider the source model $M^S$ associated to transformation~$\phi(r_1) = (r_1, r_1 + ``s")$ which consists in suffixing an ``s" at the end of a word. 
We can verify that $M^S$ is $\eta$-reusable for the target case (``film", ``films"). 
In that case, it can be verified that $M^S$ minimizes quantity $K(M) + K(x^T, y^T | M)$.
Consequently, $M^S$ is weakly $\eta$-reusable for (``film", ``films") with $\eta \leq K(M^S) - 1$. This shows that $\rho_w(M^S, (\text{``film", ``films"})) = K(M^S) - 1$. 
It can also be verified that $M^S$ is the \emph{only} model minimizing $K(M) + K(x^T, y^T | M)$. Therefore, we also have that $M^S$ is strongly $\eta$-reusable for (``film", ``films") with $\eta \leq K(M^S) - 1$. In this case, we have $\rho_s(M^S, (\text{``film", ``films"})) = \rho_w(M^S, (\text{``film", ``films"}))$. 

However, the model is not reusable for (``mouse", ``mice") for instance, since the minimal transformation for this case is $\phi^\prime(r) = (r + \text{``ouse"}, r + ``ice")$. We then have $\rho_s(M^S, (\text{``mouse", ``mice"})) = \rho_w(M^S, (\text{``mouse", ``mice"}))$.

\paragraph{Duplication.} We consider the source model $M^S$ associated to transformation~$\phi(r_1) = (r_1 + ``-" + r_1, r_1)$. This transformation is the reverse of a duplication (plural form in Indonesian). As for previous example, one can easily check that $M^S$ is not reusable for the case (``orang", ``orang-orang"). This was expected, with the choice of the transfer representation~$K(M^T | M^S)$ which forces toward reusing the source model or ignoring it completely. It would not have been the same with other choices though. In particular, if it allowed for model transformation of the form $(\phi_1, \phi_2) \mapsto (\phi_2, \phi_1)$.

\section{Toward Transferability in Domain Adaptation}
\label{sec:uda}

Domain adaptation is a machine learning task where a hypothesis on a source domain has to be transferred to a target domain where data are not equally distributed~\cite{farahani2021brief}. We conclude this paper with a quick investigation of how domain adaptation can fit to our proposed framework.  

\subsection{Related Works: Task-Relatedness}

The question of transferability of one solved source problem to a target problem has played a predominant role in the theoretical understanding of domain adaptation. Ben-David et al. (2010) \cite{ben2010theory} propose a PAC bound for transfer between domains in a binary classification setting, which relies on the use of a measure of a specific \emph{domain divergence}, called $\mathcal{H}$-divergence. It is noticeable that the introduced measure depends on the hypothesis class~$\mathcal{H}$, i.e. on the model space. This proposed notion has then been refined, to account for a variety of loss functions~\cite{mansour2009domain} or to adapt to the PAC-Bayesian setting~\cite{germain2013pac}. All these measures follow a similar idea of comparing the distribution over the input spaces, but ignore the labels. Closer to our proposal, Zhang et al. (2012)~\cite{zhang2012generalization} propose to additionally take into account the label distribution~$p(y)$ in the discrepancy measure. 
Independently from these studies, Mahmud (2009)~\cite{mahmud2009universal} also proposed to quantify the task-relatedness using Kolmogorov complexity and Algorithmic Information Theory.

\subsection{Open Question: Complexities for Probabilistic Models}

We describe the domain adaptation task in the context of the probabilistic model space defined in Example~\ref{ex:probabilistic-model}, in which a model is associated to a probability distribution. Such as for the morphological domain (Section~\ref{sec:morphology-complexities}), we need to define the complexities $K(M)$, $K(x,y | M)$ and $K(M^T | M^S)$. 

The term~$K(x,y | M)$ is a standard quantity considered by Algorithmic Information Theory. It is commonly computed using the \emph{Shannon-Fano coding}. Using this allows to define the complexity of an object $(x,y)$ knowing a probability distribution $p$ as $K(x,y | p) = - \log p(x,y)$. This corresponds to the natural choice when computing $K(x,y | M)$ in our probabilistic domain. 

Defining $K(M)$ and $K(M^T | M^S)$ is more difficult thought and goes beyond the scope of this paper. Traditionally, the complexity of a probability distribution is assimilated to the probability of its density function, and specific computations are proposed to estimate these. We refer the interest the interested reader to~\cite{mahmud2009universal} for instance.

\subsection{Discussion}

Extending our framework to domain adaptation is both natural and complex. Indeed, the problem formulation in terms of models allow for a direct characterization of domain adaptation, where a domain is characterized by an unlabeled dataset (the problem) and a vector of predictions (the solution). However, the computation of the complexities may not be as simple as for the symbolic domain. Future works will have to bridge this gap.

\section{Conclusion}

In this paper, we introduced a novel understanding of analogical transfer, by focusing on whether the information contained in the source case were transferable to the target case. We proposed a formalism based on a notion of \emph{models}, which we defined in a way that is consistent with both symbolic analogies and numerical machine learning. Even though this preliminary works cover only the question of measuring the transferability from a source case $(x^S, y^S)$ to a target case $(x^T, y^T)$, it is a first step toward the key question of predicting whether knowledge of $(x^S, y^S)$ could be reused to find a solution~$y$ to problem $x^T$. We think the proposed framework could provide a common basis to develop a general understanding of transfer, going beyond the current statistical theories~\cite{ben2010theory} for instance.

\begin{acknowledgments}
  The author wishes to thank the anonymous reviewers for their insightful comments and suggestions. Some of the presented ideas have been inspired by discussions with Antoine Cornu\'ejols, Jean-Louis Dessalles, Marie Al-Ghossein and Miguel Couceiro. This work was supported by the Academy of Finland Flagship programme: Finnish Center for Artificial Intelligence, FCAI. 
\end{acknowledgments}

\bibliography{biblio}

\end{document}